\documentclass[letterpaper, 10 pt, conference]{ieeeconf}
\usepackage[protrusion=true,expansion=true]{microtype} 
\IEEEoverridecommandlockouts                              

\usepackage[pdftex]{graphicx}
\DeclareGraphicsExtensions{.pdf}
\graphicspath{{./media/}}
\usepackage{xcolor}

\usepackage[cmex10]{amsmath}
\usepackage{amsthm}
\usepackage{amssymb}
\usepackage[ruled,vlined]{algorithm2e}
\usepackage{balance}
\usepackage{hyperref}

\usepackage{array}
\usepackage{mdwmath}
\usepackage{mdwtab}
\usepackage{eqparbox}
\usepackage{blkarray}

\usepackage[british]{babel}

\usepackage{siunitx}

\usepackage{cleveref}

\makeatletter
\let\MYcaption\@makecaption
\makeatother

\usepackage[font=footnotesize]{subcaption}

\makeatletter
\let\@makecaption\MYcaption
\makeatother

\usepackage{dblfloatfix}

\usepackage{cite}

\usepackage{url}

\newtheorem{theorem}{Theorem}

\newtheorem{lemma}{Lemma}

\theoremstyle{definition}
\newtheorem{assumption}{Assumption}
\newtheorem{definition}{Definition}

\newtheorem{remark}{Remark}

\usepackage{amssymb, amsmath, latexsym, amsfonts, mathrsfs, amsbsy}
\usepackage{color}
\usepackage{bm}

\DeclareMathOperator*{\kron}{\otimes}

\newcommand{\eqb}[1]{\begin{equation}\label{#1}}
\newcommand{\eqe}{\end{equation}}

\newcommand{\mbb}[1]{\mathbb{#1}}
\newcommand{\mc}[1]{\mathcal{#1}}

\newcommand{\ib}{\begin{itemize}}
\newcommand{\ie}{\end{itemize}}

\def \symm {+}
\newsavebox{\ieeealgbox}

\title{\LARGE \bf
Distributed Adaptive and Resilient Control of Multi-Robot Systems \\with Limited Field of View Interactions}

\author{Pratik Mukherjee$^{1*}$, Matteo Santilli$^{2}$, Andrea Gasparri$^{2}$, and Ryan K. Williams$^{1}$%
\thanks{$^{1}$P. Mukherjee and R.K. Williams are with Electrical and Computer Engineering Department, Virginia Polytechnic Institute and State University, Blacksburg, VA USA, {\tt\small \{mukhe027, rywilli1\}@vt.edu} }%
\thanks{$^{2}$M. Santilli and A. Gasparri are with the Engineering Department, Roma Tre University, Roma, 00146, Italy, {\tt\small matteo.santilli@uniroma3.it, gasparri@dia.uniroma3.it}}%
}

\def\trifov#1{\mathcal{T}_{#1}}

\def\dims{3}
\def\dimp{2}

\begin{document}

\maketitle
\thispagestyle{empty}
\pagestyle{empty}

\begin{abstract}
 In this paper, we consider two coupled problems for distributed multi-robot systems (MRSs) coordinating with limited field of view (FOV) sensors:  \emph{adaptive tuning of interaction gains and rejection of sensor attacks}.  First, a typical shortcoming of distributed control frameworks (e.g., potential fields) is that the overall system behavior is highly sensitive to the gain assigned to relative interactions.  Second, MRSs with limited FOV sensors can be more susceptible to sensor attacks aimed at their FOVs, and therefore must be resilient to such attacks.  Based on these shortcomings, we propose a comprehensive solution that combines efforts in adaptive gain tuning and attack resilience to the problem of topology  control for MRSs with limited FOVs.  Specifically, we first derive an adaptive gain tuning scheme based on satisfying \emph{nominal pairwise interactions}, which yields a dynamic balancing of interaction strengths in a robot's neighborhood.  We then model additive sensor and actuator attacks (or faults) and derive $H_{\infty}$ control protocols by employing a static output-feedback technique, guaranteeing bounded $L_2$ gains of the error induced by the attack (fault) signals.  Finally, simulation results using ROS Gazebo are provided to support our theoretical findings.

\end{abstract}

\section{Introduction}

Recent efforts in distributed multi-robot coordination have
exploded, both in number and in breadth of capability. The
significant interest in distributed systems is not surprising as
the applications are numerous, including for example sensor
networks, collaborative robotics, transportation, etc. One trend
of multi-robot control has been potential-based design. Indeed, the potential-based control
methodology has been widely investigated by the robotics/control community, e.g., \cite{santilli:2022,Gasparri:CDC:2017:1,mukherjee2019experimental,santilli2019distributed,ji2007distributed,zavlanos2007potential,dimarogonas2008connectedness,Williams:2015:ICRA,Williams:2013bh,Heintzman2018-xb,Heintzman2020-cz,Williams2017-qq,Williams2015-hw}. The advantages of potential-based design
are that controllers are easily distributed, physically-motivated,
and come with provable convergence guarantees.
Although the potential-based control framework is very powerful, it exhibits several shortcomings ranging from local minima to high sensitivity to gains. While local minima are inherent to the design of potentials encoding network objectives,
high sensitivity to gains can be mitigated with a proper choice
of pairwise interaction strength for a given multi-robot system (MRS) and its coordination objective. 
Works such as \cite{bechlioulis2014low}, \cite{bechlioulis2014robust} (which are based on the prescribed
performance framework) and \cite{karayiannidis2012multi} control the transient response of an error signal
in formation control and consensus to satisfy a prescribed system performance, yielding a gain tuning behavior. \cite{bechlioulis2014robust} achieves robust model-free prescribed control for leader-follower formations, which is decoupled from graph topology,
control gain selection, and model uncertainties. Alternatively,
works such as \cite{korda2014stochastic}, \cite{guo2013controlling} take a model-driven approach in controlling constraint violation, \cite{korda2014stochastic} applying model-predictive
control and \cite{guo2013controlling} through navigation functions (a form of potential control), both in a centralized context. On the contrary, we develop a \emph{distributed} adaptive controller, which is resilient to external attacks or disturbances, for topology control applications of MRS where the robots are equipped with \emph{limited field of view} (FOV) sensors.

Control with limited FOV \cite{Asadi:2016,Zhang:2013} is an interesting yet less explored topic of research, e.g., for aerial robots equipped with cameras capable of sensing other robots in a visually restricted area. 
In our recent work \cite{santilli2019distributed}, we exhibit a distributed potential-based coordination framework for an MRS equipped with limited FOV sensors to maintain an asymmetric, connected topology (i.e., a weakly connected directed graph). In such a context, threat from external attacks or disturbances can destabilize an already sensitive system. Limited FOV control is not just prone to the sensitivity of potential-based controllers, but also the threat of external attack, especially as stability is not intrinsic for asymmetric interactions \cite{Gasparri:CDC:2017:1,mukherjee2020optimal}. Therefore, it is important that we develop a control scheme that is resilient to such external attacks and disturbances. Works like \cite{chen2020adaptive,chen2019resilient} address the problem of solving for resilient controllers for MRSs that are exposed to external attacks, faults, or disturbances. 
The authors use the \emph{static output feedback} technique to show that the error induced from faults has bounded $L_2$ gain in terms of the $L_2$ norm of fault signals. 
In this paper, we integrate the static output feedback technique using $H_{\infty}$ control protocols to fit our potential-based control framework in the context of limited FOV interactions (as in \cite{santilli2019distributed}).

In the case of the potential-based control framework, when considering more complex control objectives like enforcing desirable topological conditions (as opposed to consensus or formation control), especially with limited FOVs, the state of the art is still a non-linear potential field as their theoretical analysis remains tractable.  However, the problem remains of actually deploying potential-based methods on real robots as they are highly sensitive to appropriate selection of control gains and external attacks or disturbances. Thus, in this paper we aim to provide a comprehensive solution to applying potential-based design for real-world MRSs by combining efforts in adaptive gain tuning and attack or disturbance resilience to the problem of topology control with limited FOVs (as depicted in Figure \ref{fig:gaz_mot}). Our vision is that our methods can reduce the gap between theoretically tractable multi-robot control methods and real-world deployments.
In this direction, our contribution is threefold;
i) We demonstrate that our distributed controller yields
an approximation of a desired pairwise interaction strength between all robots,
while preserving the equilibria of an underlying objective for robots coordinating in asymmetric interaction scenarios with limited FOVs. 
 ii) We couple the potential-based controller with $H_{\infty}$ control protocols to reject external attacks or disturbances on an MRS. 
 iii) We corroborate the theoretical results with extensive ROS Gazebo simulations.
 
 The remainder of the paper is organized as follows. In Section~\ref{sec:preliminaries}, preliminary notions are given. In Section~\ref{sec:control}, a brief review of limited FOV topology control from \cite{santilli2019distributed} is done. In Section~\ref{sec:dist_adp_sec}, the proposed coordination framework with \emph{adaptive} control is derived. In Section~\ref{sec:res_sec}, the derivation of \emph{resilient} adaptive control is presented. In Section~\ref{sec:simulations}, the results of numerical and experimental validations are described. Finally, in Section~\ref{sec:conclusion}, we conclude our work.

\section{Preliminaries}\label{sec:preliminaries}

Let us consider an MRS composed of $n$~robots and assume that each robot~$i$ has a first-order dynamics $\dot{s}_i(t) = u_i(t)$
with $s_i(t) = [p_i(t)^T, \, \theta_i(t)]^T \in \mathbb{R}^{3}$ the state of robot~$i$ composed of the position $p_i(t) = [x_i(t), y_i(t)]^T \in \mathbb{R}^{\dimp}$ and the orientation $\theta_i \in (-\pi,\pi]$, while $u_i(t) \in \mathbb{R}^{\dims}$ denotes the control input. Stacking robot states and inputs yields the overall system $\dot{\mathbf{s}}(t) = \mathbf{u}(t)$
with $\mathbf{s}(t) = [s_1(t)^T,\, \ldots, \, s_n(t)^T]^T \in \mathbb{R}^{3 n} $ and $\mathbf{u}(t) = [u_1(t)^T,\, \ldots, \, u_n(t)^T]^T \in \mathbb{R}^{3 n}$ the stacked vector of states and control inputs, respectively. In the sequel, time-dependence will be omitted for the sake of brevity.

Let us assume that each robot~$i$ possesses a limited field of view that is encoded by a triangle geometry $\trifov{i}$ rigidly fixed to the robot as depicted in Figure \ref{fig:gaz_mot} \cite{santilli2019distributed}. This kind of sensing yields asymmetric robot interactions that we will describe through a directed graph \mbox{$\mathcal{G} = \{ \mathcal{V}, \, \mathcal{E} \}$} with node set $ \mathcal{V} = \{ 1,\ldots,n \}$ and edge set $\mathcal{E} \subseteq \mathcal{V} \times \mathcal{V}$. In particular, we will say that an edge $e_{ij} \in \mathcal{E}$ connects robot~$i$ and robot~$j$ if $p_j \in \trifov{i}$. In addition, when referencing single edges we will use the convention $e_k$ meaning that we are referencing the $k$-th directed edge out of $|\mathcal{E}|$ total edges\footnote{In order to reference the $k$-th edge, the edge set $\mathcal{E}$ needs to be sorted. A simple sort can be obtained enumerating the edges $(1,j)$ of the first robot as \mbox{$e_1,e_2,\ldots,$} then the edges $(2,h)$ of the second robot and so on.}. Moreover, we will denote by \mbox{$\mathcal{N}_{i}^{+} = \{j \in \mathcal{V} :(i,j) \in \mathcal{E}\}$} the set of \emph{out-neighbors} of robot~$i$ and \mbox{$\mathcal{N}_{i}^{-} = \{j \in \mathcal{V} :(j,i) \in \mathcal{E}\}$} the set of \emph{in-neighbors}. Note that since the graph is directed, $(i,j) \in \mathcal{E}$ does not imply $(j,i) \in \mathcal{E}$. 
A useful representation for a directed graph $\mathcal{G}$ is the \textit{incidence matrix} $ \mathcal{B}(\mathcal{G})\in \mathbb{R}^{n\times |\mathcal{E}|}$, that is a matrix with rows indexed by robots and columns indexed by edges, such that $\mathcal{B}_{ij} = 1$ if the edge $e_j$ leaves vertex $v_i$, $-1$ if it enters vertex $v_i$, and $0$ otherwise. The \emph{outgoing incidence matrix} $\mathcal{B}_+$ contains only the outgoing parts of the incidence matrix $\mathcal{B}$, with incoming parts set to zero. We will also make use of the \emph{directed edge Laplacian} $\mathcal{L}_{\mathcal{E}}^d \in \mbb{R}^{|\mc{E}| \times |\mc{E}|}$ given by $\mc{L}_{\mc{E}}^d = \mc{B}^T\mc{B}_+$ (see \cite{Zelazo:2007,Zeng:2016}).  

\begin{figure}[t!]
\centering
    \includegraphics[width=0.4\textwidth]{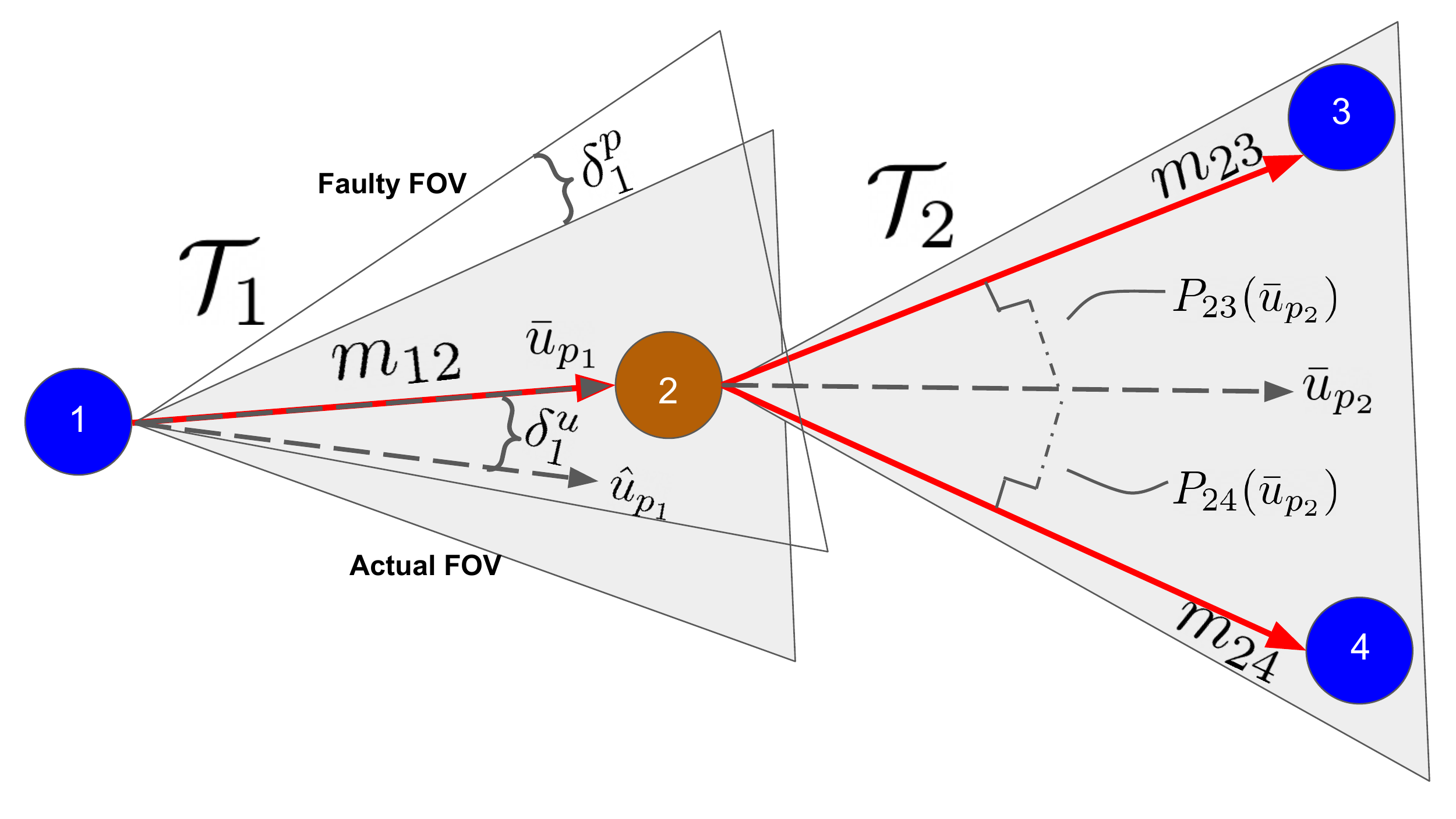}
    \caption{An illustration of the geometry of adaptive controller with faulty sensor and actuator. The interaction, limited FOV $\mathcal{T}_1$, between robots $1$ and $2$ is marred by sensor fault $\delta^p_1$ and actuator fault $\delta^u_1$ which leads to the deviation of the interaction from the desired interaction $m_{12}$. The \emph{resilient} interactions, $\mathcal{T}_2$, between robots $2\rightarrow3$ and $2\rightarrow4$ depict the desired interactions $m_{23}$ and $m_{24}$ with the projections $P_{23}$ and $P_{24}$ of the distributed controller $\Bar{u}_{p_2}$ onto (two of) the model interactions as depicted by dashed lines. The overall goal is to select  gains for the distributed controller that minimize the model fitting error. }
    \label{fig:gaz_mot}
\end{figure}

\section{Directed Coordination Framework}\label{sec:control}

In this section, we review the theory of potential-based control framework for MRS with limited FOV from our prior work \cite{santilli2019distributed}. As a case study let us consider the maintenance of the interactions (topology) among robots in a network, i.e., we want to preserve the initial graph~$\mathcal{G}$.  In other words, we want to derive a distributed potential-based control law such that each robot~$i$ maintains its neighbor~$j$ inside the triangle~$\trifov{i}$ that defines the perimeter of the FOV of each robot as detailed in \cite{santilli2019distributed}. In this direction, we can define the potential field term $\Phi_{ij}(s_i,s_j)$ which encodes the sum of energies of the perpendicular distances from a point, the neighboring robot~$j$, to the line, the sides of the triangle~$\trifov{i}$.

Next, we model a quality of interaction map  by using a two-dimensional Gaussian function $\Psi_{ij}(s_i,s_j)$ with mean  and variance that encodes the desired location of neighboring robot $j$ in robot $i$'s FOV.
For the FOV maintenance controller, we can move along the anti-gradient of our potential fields $\Phi_{ij}(s_i,s_j)$ and $\Psi_{ij}(s_i,s_j)$, a typical potential-based control procedure, to orient the FOV of the robots in order to minimize the distance between the neighbors $j \in \mathcal{N}^+_i$ and the desired position. 
Therefore, we can introduce the control law that each robot~$i$ needs to run in order to keep its neighbors in the limited sensing zone $\trifov{i}$ near the desired point as the following:
\begin{equation}\label{eq:u_i}
\begin{aligned}
\dot{s}_i &= - \underbrace{\sum \limits_{j \in \mathcal{N}_i^+} \nabla_{s_i} \Big ( \Phi_{ij}(s_i,s_j)+\Psi_{ij}(s_i,s_j) \Big )}_{u_i} \\
\end{aligned}
\end{equation}

\subsection{Stability Analysis of Asymmetric Interaction Graphs}

In this section we provide the main result of our work from \cite{santilli2019distributed}, that is the stability analysis of the topology control framework reviewed above. From \cite{santilli2019distributed} we know that the overall Lyapunov function  for directed graphs $\Bar{V}: \mathbb{R}^{3 n} \rightarrow \mathbb{R}_+$ is defined as
\begin{equation}\label{eq:V}
\begin{aligned}
    \Bar{V} (\mathbf{s}(t)) 
    &= \sum \limits_{i=1}^n \sum \limits_{j \in \mathcal{N}_i^+} \underbrace{\Big ( \Phi_{ij} \left (s_i,s_j \right ) + \Psi_{ij} \left (s_i, s_j \right) \Big )}_{\Bar{V}_{ij} \left (s_i,s_j \right )}
\end{aligned}
\end{equation}
Furthermore, let us introduce the matrix $\overline{\mathcal{L}}$ defined as
\begin{equation}\label{eq:overlineL}
    \overline{\mathcal{L}} = \begin{bmatrix}
    \left ( \mathcal{B}^T \mathcal{B}_+ \right )\otimes I_{\dimp} & \mathbf{0}_{ \dimp |\mathcal{E}| \times |\mathcal{E}|} \\
    \mathbf{0}_{|\mathcal{E}| \times \dimp |\mathcal{E}|} & \left( \mathcal{B}_+^T \mathcal{B}_+ \right )
\end{bmatrix}
\end{equation}
where 
$I_{\dimp}$ is the $\dimp \times \dimp$ identity matrix and $\mathbf{0}_{r \times \Bar{r}}$ is a $r \times \Bar{r}$ zeros matrix. We are now ready to recall our main result.

\begin{theorem}(Theorem 1 in \cite{santilli2019distributed})\label{th:1}
Consider the multi-robot system $\dot{\mathbf{s}}(t) = \mathbf{u}(t)$ running control laws \eqref{eq:u_i}. Then, if the symmetric part of matrix $\overline{\mathcal{L}}$ in \eqref{eq:overlineL}, $ \overline{\mathcal{L}}^{\symm} = \frac{1}{2} \left( \overline{\mathcal{L}} + \overline{\mathcal{L}}^T \right )$, is positive semi-definite, the system is stable in the sense that if the energy $\Bar{V} (\mathbf{s}(t))$ is finite at time $t = t_0$ then it remains finite for all $t > t_0$.
\end{theorem}

From \cite{santilli2019distributed}, we show that the Lyapunov function $\Bar{V}(s)$ is continuously differentiable for which we can define the level set~$\Delta_c = \{s : \Bar{V}(s) \leq c \}$ for any $c>0$ as a compact and invariant set with respect to the relative position of the robots. Arguments about the compactness of the level sets~$\Delta_c$ with respect to the relative distances can be found in \cite{Dimarogonas:2008,Tanner:TAC:2007}. LaSalle's principle now guarantees that the system will converge to the largest invariant subset of $\{ s :  \dot{\Bar{V}}(s) = 0 \}$ by analyzing the nullspace of the matrix~$\overline{\mathcal{L}}^{\symm}$, $\left \{\xi \in \mathbb{R}^{3|\mathcal{E}|} : \overline{\mathcal{L}}^{\symm} \xi = 0 \right \}$ for the Lyapunov derivative defined as
\begin{equation}\label{eq:dotV:2}
\begin{aligned}
    \dot{\Bar{V}}(\mathbf{s}) 
     =& - \frac{1}{2} \xi^T \! \left [ \left( \overline{\mathcal{L}}\!  +\!  \overline{\mathcal{L}}^T \right ) \! + \! \left( \overline{\mathcal{L}}\!  -\!  \overline{\mathcal{L}}^T \right )  \right ] \xi =- \xi^T \overline{\mathcal{L}}^{\symm} \xi \leq 0 \\
\end{aligned}
\end{equation}
where $\overline{\mathcal{L}}$ is already defined in \eqref{eq:overlineL} and  
using the stacked vector of potential field gradients $\xi \in \mathbb{R}^{\dims|\mathcal{E}|} = \left [\xi_{xy}^T, \, \xi_{\theta}^T \right]^T$ where $\xi_{xy} \in \mathbb{R}^{2|\mathcal{E}|}$ and $\xi_{\theta} \in \mathbb{R}^{|\mathcal{E}|}$ are defined as
\begin{equation}\label{eq:xi}
\begin{aligned}
&\xi_{xy}=\!  \resizebox{0.92\hsize}{!}{%
    $\left [ \nabla_{x_{e_{1}(1)}} \Bar{V}_{e_1}^T, \nabla_{y_{e_{1}(1)}} \Bar{V}_{e_1}^T, \ldots, \nabla_{x_{e_{|\mathcal{E}|}(1)}} \Bar{V}_{e_{|\mathcal{E}|}}^T, \nabla_{y_{e_{|\mathcal{E}|}(1)}} \Bar{V}_{e_{|\mathcal{E}|}}^T \right ]^T$
    } \\
&\xi_{\theta} = \resizebox{.5\hsize}{!}{
$\left [ \nabla_{\theta_{e_{1}(1)}} \Bar{V}_{e_1}^T, \ldots, \nabla_{\theta_{e_{|\mathcal{E}|}(1)}} \Bar{V}_{e_{|\mathcal{E}|}}^T \right ]^T$}
\end{aligned}
\end{equation}
where $e_k(1)$ denotes the starting vertex of the $k$-th edge $(i,j)$, and thus $\nabla_{w_{e_{k}(1)}} \Bar{V}_{e_k} \in \mathbb{R}$ denotes the gradient with respect to the state variable $s_i \in \{x_i, y_i, \theta_i \}$ of potential function $\Bar{V}_{ij}$.
The Lyapunov time derivative in \eqref{eq:dotV:2}  is in a typical quadratic form and hence its negative-semidefiniteness depends on the positive-semidefiniteness of symmetric block diagonal matrix~$\overline{\mathcal{L}}^{\symm}$.

\section{Distributed Adaptive FOV Control}\label{sec:dist_adp_sec}
In the previous section, we provided a stable control law \eqref{eq:u_i} originally developed in \cite{santilli2019distributed} to conduct a stable topology control for an MRS with limited FOV. 
In this section, we modify the control law to adaptively tune the gains for an MRS to improve robustness in realistic deployments. 
\subsection{Lyapunov-Based Adaptive Control Design}\label{sub_sec:Lyp_adap}
For the sake of analysis, let us introduce a bijective indexing function $g(\cdot): \mathcal{V}\times \mathcal{V}\rightarrow \{1,...,|\mathcal{E}|\}$ to sort directed edges, that is $g(\mathcal{E})=[e_1,...,e_{|\mathcal{E}|}]$. Note that, since the graph is directed we have $e_h=g(e_{ij})\neq g(e_{ji}) $. Clearly, it is possible to define the inverse function such that if $e_h = g(e_{ij})$ then $e_{ij}=g^{-1}(e_h)$. Let us now consider an additional
stacked vector of gains $\mathbf{k}$ defined as
\begin{equation}\label{eq:k_vec}
\begin{aligned}
    \mathbf{k}=[k_{e_1},...,k_{e_{|\mathcal{E}|}}]
\end{aligned}    
\end{equation}
Then we have the following generalization of the distributed control law in \eqref{eq:u_i} for a given robot $i$:
\begin{equation}\label{eq:gen_con_law}
\begin{aligned}
    \dot{s}_i = \underbrace{- \sum_{j \in \mathcal{N}_i^+} k_{ij} \nabla_{s_i} \Bar{V}_{ij}(s_i,s_j)}_{\Bar{u}_i (\mathbf{s},\mathbf{k})}\\
    \dot{k}_{ij}= u_{ij}(\mathbf{s},\mathbf{k}), \quad \forall \; e_{ij} \;|\; j \in \mathcal{N}_i^+
\end{aligned}    
\end{equation}
where $\Bar{u}_i=[\Bar{u}_{p_i}^T,\Bar{u}_{\theta_i} ]^T$ and $u_{ij}(\mathbf{s},\mathbf{k})$, the gain control laws, will be defined later. Our goal is to design for each pair of robots $i$ and $j$ the control term $u_{ij}(\mathbf{s},\mathbf{k})$ to adaptively tune pairwise gains $k_{ij}$ to approximate the nominal pairwise interaction model given by $\Bar{V}_{ij}(s_i,s_j)$, independent of the network size. To design such an adaptive potential-based control law the following Lyapunov function can be considered 
\begin{equation}\label{eq:adp_V}
\begin{aligned}
    V(\mathbf{s},\mathbf{k})=\underbrace{\sum^n_{i=1}\sum_{j\in \mathcal{N}_i^+}k_{ij}\Bar{V}_{ij}(s_i,s_j)}_{\hat{V}(\mathbf{s},\mathbf{k})} + \underbrace{\sum^n_{i=1}\sum_{j\in \mathcal{N}_i^+} F_{ij}(\mathbf{p},\mathbf{k})}_{F(\mathbf{p},\mathbf{k})}
\end{aligned}    
\end{equation}
with $F_{ij}(\mathbf{p},\mathbf{k})$ the \emph{pairwise cost objective function} encoding the deviation of each pairwise interaction from its nominal model defined as
\begin{equation}\label{eq:adp_F}
\begin{aligned}
   F_{ij}(\mathbf{p},\mathbf{k})= \frac{1}{2} \| P_{ij}(\Bar{u}_{p_i}(\mathbf{p},\mathbf{k}))- m_{ij}(\mathbf{p},\mathbf{k})\|^2
\end{aligned}    
\end{equation}
with our \emph{desired pairwise nominal model} $m_{ij}$ defined as
\begin{equation}\label{eq:mij}
\begin{aligned}
   m_{ij}(\mathbf{p},\mathbf{k})= - \nabla_{p_i} \Bar{V}_{ij}(p_{ij})
\end{aligned}    
\end{equation}
that is, the anti-gradient of the potential function $\Bar{V}_{ij}$ with respect to robot $i$, where $P_{ij}(\Bar{u}_{p_i}(\mathbf{p},\mathbf{k}))$ is the projection of the control input of robot $i$ over the line of sight of the robots $i$ and $j$ defined as 
\begin{equation}\label{eq:Pij}
\begin{aligned}
   P_{ij}(\Bar{u}_{p_i}(\mathbf{p},\mathbf{k}))=\left ( \frac{p_{ij}^T \Bar{u}_{p_i}(\mathbf{p},\mathbf{k})}{\|p_{ij}\|^2} \right ) p_{ij}
\end{aligned}    
\end{equation}
Note, we make $F(\mathbf{p},\mathbf{k})$ explicitly a function of the state variables $\mathbf{p}=[\mathbf{x}^T,\mathbf{y}^T]^T$ because intuitively $\theta$ is embedded in the computation of the projection operator in \eqref{eq:Pij} since the per-robot terms $x_{ij}$ and $y_{ij}$ are implicitly a function of $\theta$ in the polar coordinate form $(\varphi, \theta )$ where $x_{ij}=\varphi \cos\theta$ and $y_{ij}=\varphi \sin\theta$ where $\varphi$ is the radian distance between robot $i$ and $j$. 

With this intuition, we can complete the control law given in \eqref{eq:gen_con_law} by introducing the following distributed control law for each edge-wise gain $k_{ij}$ with $(i,j)\in \mathcal{E}$
\begin{equation}\label{eq:uij}
\begin{aligned}
 \dot{k}_{ij}=u_{ij}(\mathbf{s},\mathbf{k})= - \nabla_{k_{ij}} F(\mathbf{p},\mathbf{k}) + w_{ij}
\end{aligned}    
\end{equation}
where the term $\nabla_{k_{ij}}F(\mathbf{p},\mathbf{k})$ is computed per edge in the directed graph as
\begin{equation}\label{eq:del_F}
\begin{aligned}
\nabla_{k_{ij}}F(\mathbf{p},\mathbf{k}) = \sum_{(i,h) \in\mathcal{E}} \nabla_{k_{ij}}F_{ih}(\mathbf{p},\mathbf{k}) 
\end{aligned}    
\end{equation}
Therefore, the overall $\nabla_{\mathbf{k}}F(\mathbf{p},\mathbf{k})$ is given as the stacked vector of potential fields
$\nabla_{\mathbf{k}}F(\mathbf{p},\mathbf{k})= [ \nabla_{k_{e_1}}F(\mathbf{p},\mathbf{k})^T,...,\nabla_{k_{e_{|\mathcal{E}|}}}F(\mathbf{p},\mathbf{k})^T]^T $. An alternative form is shown below as
\begin{equation}\label{eq:del_F2}
\begin{aligned}
\nabla_{\mathbf{k}}F(\mathbf{p},\mathbf{k})= [ \sum_{(i,h)\in \mathcal{E}} \nabla_{k_{e_{1}^{(i)}}}F_{ih}(\mathbf{p},\mathbf{k})^T,..., \\ 
\sum_{(i,h)\in \mathcal{E}}\nabla_{k_{e_{|\mathcal{E}|}^{(i)}}}F_{ih}(p,\mathbf{k})^T]^T 
\end{aligned}    
\end{equation}
  where $e_k^{(i)}$ is the $k^{th}$ edge $(i,h) \in \mathcal{E}$ with starting vertex $i$ (i.e. only outgoing directed edges). Note that the contribution to $\nabla_{k_{ij}}F(\mathbf{p},\mathbf{k})$ is from the directed edges only for all $(i,j)\in \mathcal{E}$ and the term $w_{ij}$.
The form of term $w_{ij}$ is the result of the Lyapunov-based design and is required to enforce the negative semi-definiteness of $\dot{V}(\mathbf{s},\mathbf{k})$, which will be derived in Theorem \ref{th:gain_stab}. The interactions $2\rightarrow3$ and $2\rightarrow4$ in Figure~\ref{fig:gaz_mot} depict the idea of desired interactions.
  
\subsection{Adaptive Control Stability Analysis }

Let us now derive the overall extended dynamics of the MRS where each robot evolves according to \eqref{eq:gen_con_law} where gains are updated according to \eqref{eq:uij}. In particular, we obtain the following control laws
\begin{equation}\label{eq:adp_law}
\begin{aligned}
\dot{\mathbf{s}}&=-\nabla_{\mathbf{s}+}\hat{V}(\mathbf{s},\mathbf{k}) \quad \dot{\mathbf{k}}&= -\nabla_{k}F(\mathbf{p},\mathbf{k}) + \mathbf{w}
\end{aligned}    
\end{equation}
where $\mathbf{s}+$ in $\dot{\mathbf{s}}$ (expanded form in eq.(29) in \cite{santilli2019distributed}) indicates that it has contributions only from the starting vertex of
each edge\footnote{Refer to \cite{santilli2019distributed} for complete structure of $\dot{\mathbf{s}}$. Also note that $\dot{\mathbf{p}}$ can be represented in the same way as $\dot{\mathbf{s}}$ as it consists of a subset of state variables of $\dot{\mathbf{s}}$.} and the stacked vector $\mathbf{w}$ is defined similarly to the stacked vector $\mathbf{k}$ as follows 
\begin{equation}\label{eq:w}
\begin{aligned}
\mathbf{w}=[w_{e_1},...,w_{e_{|\mathcal{E}|}}]
\end{aligned}    
\end{equation}
that is a vector collecting the terms $\{w_{ij}\}$ with $(i,j)\in \mathcal{E}$.
In order to demonstrate the theoretical properties of the adaptive control law given in \eqref{eq:adp_law}, the following convexity result of the cost $F_{ij}(\mathbf{p},\mathbf{k})$ is given.

\begin{theorem} \label{th:convx_F}
The potential function $F( \mathbf{p}, \mathbf{k})$ given as a sum
of pairwise cost $F_{ij} (\mathbf{p}, \mathbf{k})$ encoding the deviation of each pairwise interaction from its nominal model is convex in $\mathbf{k}$ for a fixed $\mathbf{p}$.
\end{theorem}

\begin{proof}
First, we rewrite $P_{ij}(\Bar{u}_{p_i}(\mathbf{p},\mathbf{k}))$ as
\begin{equation}\label{eq:Pij2}
\begin{aligned}
  P_{ij}(\Bar{u}_{p_i}(\mathbf{p},\mathbf{k}))=\underbrace{\frac{1}{\|p_{ij}\|^2} (p_{ij} p_{ij}^T) }_{A_{ij}} \Bar{u}_{p_i}
  \end{aligned}    
\end{equation}
with $A_{ij}$ as a symmetric, positive semidefinite ``projection" matrix with spectrum $\sigma(A_{ij})=\{0,1\}$. Therefore, we can rewrite the function $F_{ij}(\mathbf{p},\mathbf{k})$ in \eqref{eq:adp_F} as
\begin{equation}\label{eq:new_Fij}
\begin{aligned}
  F_{ij}(\mathbf{p},\mathbf{k}) = \frac{1}{2}\|A_{ij}B_i \mathbf{k} - m_{ij}  \|^2
  \end{aligned}    
\end{equation}
where the matrix $B_i \in \mathbb{R}^{2\times |\mathcal{E}|}$ is defined as 
\begin{equation}\label{eq:Bi}
\begin{aligned}
B_i=[0,...,-\nabla_{p_i}\Bar{V}_{ij},0,...,-\nabla_{p_i}\Bar{V}_{ik},...,0]
  \end{aligned}    
\end{equation}
where all $-\nabla_{p_{i}}\Bar{V}_{ij}$ for which $j\in \mathcal{N}_i^+$ appear. 
At this point we compute the Hessian of $F_{ij}$ as defined in \eqref{eq:new_Fij} with respect to $\mathbf{k}$ for a fixed $\mathbf{p}$ to obtain  
\begin{equation}\label{eq:Hessian}
\begin{aligned}
H_{\mathbf{k}}= (A_{ij}B_i)^T (A_{ij}B_i) = B_i^T A_{ij}^T A_{ij} B_i
  \end{aligned}    
\end{equation}
Since $A_{ij}$ is symmetric positive semidefinite, we can write $A_{ij}^2=Q^T\Lambda Q$ for a unitary matrix $Q$ and $\Lambda = \textbf{diag}\{0,1\}$ to obtain the form
\begin{equation}\label{eq:Hessian2}
\begin{aligned}
H_{\mathbf{k}}&=B_i^T A_{ij}^2 B_i = 
 (QB_i)^T \Lambda (QB_i)= Y^T \Lambda Y
  \end{aligned}    
\end{equation}
from which we notice that the construction of the hessian $H$ has spectrum $\sigma (H)= \{0,...,0,\sum_{i=1}^{|\mathcal{E}|}\Bar{y}^2_i\}$ where $Y=[\Bar{y}_i,...,\Bar{y}_{|\mathcal{E}|}]$, thus proving the convexity of $F_{ij}$.
To conclude, by invoking the sum-rule of convex functions \cite{boyd2004convex}, the result follows.
\end{proof}

We are now ready to state our main result on the stability and
convergence of the proposed adaptive potential-based control
law given in \eqref{eq:adp_law}.
\begin{theorem}\label{th:gain_stab}
Consider an MRS running \eqref{eq:adp_law} designed according to a generalized potential function as in \eqref{eq:adp_V} with $w_{ij}$
\begin{equation}\label{eq:wij2}
\begin{aligned}
w_{ij}= \frac{1}{\alpha_{ij}}(\gamma_{ij} + \frac{1}{|\mathcal{N}_i^+| + |\mathcal{N}_i^-|} \beta_{i} + \frac{1}{|\mathcal{N}_j^+| + |\mathcal{N}_j^-|} \beta_{j})
\end{aligned}    
\end{equation}
where $\alpha_{ij},\gamma_{ij},\beta_{i}$, and $\beta_{j}$ are introduced in \eqref{eq:psd_cond_edg}. Then the MRS converges to a final state $[\mathbf{s}_f^T,\mathbf{k}_f^T]$ reaching one of the critical points of $\hat{V}(\mathbf{s},\mathbf{k}_f)$ and the global optimum of $F(\mathbf{p}_f,\mathbf{k})$.
\end{theorem}

\begin{proof}
Let us consider the Lyapunov function given in \eqref{eq:adp_V} and let us compute the time derivative as follows:
\begin{equation}\label{eq:adp_vdot}
\begin{aligned}
\dot{V}(\mathbf{s},\mathbf{k}) &= \nabla_{\mathbf{s}} \hat{V}(\mathbf{s},\mathbf{k})^T \dot{\mathbf{s}} + \nabla_{\mathbf{k}} \hat{V}(\mathbf{s},\mathbf{k})^T \dot{\mathbf{k}}\\
&+ \nabla_{\mathbf{p}} F(\mathbf{p},\mathbf{k})^T \dot{\mathbf{p}}+\nabla_{\mathbf{k}} F(\mathbf{p},\mathbf{k})^T \dot{\mathbf{k}}\\
&= - \xi_{\mathbf{k}}^T \, \overline{\mathcal{L}}^+ \, \xi_{\mathbf{k}} - \| \nabla_{\mathbf{k}}F(\mathbf{p},\mathbf{k})^T\|^2\\
&- \nabla_{\mathbf{p}}F(\mathbf{p},\mathbf{k})^T\nabla_{\mathbf{p}
+}\hat{V}(\mathbf{p},\mathbf{k})\\
&-\nabla_{\mathbf{k}}\hat{V}(\mathbf{s},\mathbf{k})^T \nabla_{\mathbf{k}}F(\mathbf{p},\mathbf{k})\\
&+(\nabla_{\mathbf{k}}\hat{V}(\mathbf{s},\mathbf{k})^T+ \nabla_{\mathbf{k}}F(\mathbf{p},\mathbf{k})^T)
\mathbf{w}
\end{aligned}    
\end{equation}
Before we delve further into the derivation of the above, we would like to present the different structures of the components in the first line of equation \eqref{eq:adp_vdot} for better clarity.  From equations $(25)$, $(29)$ in \cite{santilli2019distributed} and equation \eqref{eq:adp_law}, we already know the structures of $\nabla_s \hat{V}(\mathbf{s},\mathbf{k})$, $\dot{\mathbf{s}}$, $\dot{\mathbf{p}}$ or $\nabla_{\mathbf{p}+}\hat{V}(\mathbf{p},\mathbf{k})$ and $\dot{\mathbf{k}}$, respectively. Equations \eqref{eq:del_F} or \eqref{eq:del_F2} provide the structure for $\nabla_{\mathbf{k}} F(\mathbf{p},\mathbf{k})$.  Below, we show the structure of $\nabla_{\mathbf{p}} F(\mathbf{p},\mathbf{k})$ .
\begin{equation}\label{eq:gradF-derivation:1}
\begin{aligned}
    \nabla_{\mathbf{p}}F &=  
     \begin{aligned}  \Big [
    \sum \limits_{j \in \mathcal{N}_1^+} \nabla_{p_1} F_{1j}^T &+  \sum \limits_{j \in \mathcal{N}_1^-} \nabla_{p_1} F_{j1}^T \hspace{3mm}, \,\, \ldots \, \, ,  \\
    &+ \sum \limits_{j \in \mathcal{N}_n^+} \nabla_{p_n} F_{nj}^T + \sum \limits_{j \in \mathcal{N}_n^-} \nabla_{p_n} F_{jn}^T \Big ]^T
    \end{aligned}  \\
\end{aligned}
\end{equation}
where with the gradient contributions (each element of the above vector) related to any robot~$h$, we have
\begin{equation}\label{eq:gradF-derivation:2}
\begin{aligned}
    &\nabla_{p_h} F = \sum \limits_{j \in \mathcal{N}_h^+} \nabla_{p_h} F_{hj} + \sum \limits_{j \in \mathcal{N}_h^-} \nabla_{p_h} F_{jh} \\
    &= \begin{bmatrix}
    \sum \limits_{j \in \mathcal{N}_h^+} \nabla_{x_h} F_{hj} \\
    \sum \limits_{j \in \mathcal{N}_h^+} \nabla_{y_h} F_{hj}) \\
    \end{bmatrix} 
    -  \begin{bmatrix}
    \sum \limits_{j \in \mathcal{N}_h^-} \nabla_{x_j} F_{jh} \\
    \sum \limits_{j \in \mathcal{N}_h^-} \nabla_{y_j} F_{jh} \\
\end{bmatrix}
\end{aligned}
\end{equation}
From \eqref{eq:gradF-derivation:2} we know  $\nabla_{\mathbf{p}}F$ is similar to $\nabla_{\mathbf{s}}\hat{V}$(from \cite{santilli2019distributed}). 

Finally, the structure of $\nabla_{\mathbf{k}} \hat{V}$ is given as
\begin{equation}\label{eq:delk_V2}
\begin{aligned}
\nabla_{k}\hat{V}(\mathbf{s},\mathbf{k})= 
 [  \nabla_{k_{e_{1}}}\hat{V}(\mathbf{s},\mathbf{k})^T,...,\nabla_{k_{e_{|\mathcal{E}|}}} \hat{V}(\mathbf{s},\mathbf{k})^T]^T \\ 
=[  \nabla_{k_{e_{1}^{(i)}}}\hat{V}_{ij}(\mathbf{s},\mathbf{k})^T,...,\nabla_{k_{e_{|\mathcal{E}|}^{(i)}}}\hat{V}_{ij}(\mathbf{s},\mathbf{k})^T]^T 
\end{aligned}    
\end{equation}
Now that we have identified different components of \eqref{eq:adp_vdot}, it follows that in order to ensure negative-semidefiniteness of the Lyapunov derivative in \eqref{eq:adp_vdot} the following must hold:
\begin{equation}\label{eq:psd_cond}
\begin{aligned}
&- \nabla_{\mathbf{p}}F(\mathbf{p},\mathbf{k})^T\nabla_{\mathbf{p}+}\hat{V}(\mathbf{p},\mathbf{k})
-\nabla_{\mathbf{k}}\hat{V}(\mathbf{s},\mathbf{k})^T \nabla_{\mathbf{k}}F(\mathbf{p},\mathbf{k})\\
&+(\nabla_{\mathbf{k}}\hat{V}(\mathbf{s},\mathbf{k})^T+ \nabla_{\mathbf{k}}F(\mathbf{p},\mathbf{k})^T)
\mathbf{w}\leq 0
\end{aligned}    
\end{equation}
Therefore, to satisfy relation \eqref{eq:psd_cond} and consequently the negative semi-definiteness of $\dot{V}$, from \eqref{eq:psd_cond}, we derive a per-edge form:
\begin{equation}\label{eq:psd_cond_edg}
\begin{aligned}
&\sum_{(i,j)\in \mathcal{E}}\left (\underbrace{\nabla_{k_{ij}}\hat{V}(\mathbf{s},\mathbf{k})^T+\nabla_{k_{ij}}F(\mathbf{p},\mathbf{k})^T}_{\alpha_{ij}}\right )w_{ij}   \leq   \\
&\sum_{(i,j)\in \mathcal{E}} \underbrace{\nabla_{k_{ij}} \hat{V}(\mathbf{s},\mathbf{k})^T \nabla_{k_{ij}} F(\mathbf{p},\mathbf{k})}_{\gamma_{ij}} +\\
& \sum_{(i,j)\in \mathcal{E}} ( \frac{1}{|\mathcal{N}_i^+| + |\mathcal{N}_i^-|} \underbrace{\nabla_{p_i} F(\mathbf{p},\mathbf{k})^T \nabla_{p+_i}\hat{V}(\mathbf{p},\mathbf{k})}_{\beta_i}+\\
& \frac{1}{|\mathcal{N}_j^+| + |\mathcal{N}_j^-|} \underbrace{\nabla_{p_j} F(\mathbf{p},\mathbf{k})^T \nabla_{p+_j}\hat{V}(\mathbf{p},\mathbf{k})}_{\beta_j})
\end{aligned}    
\end{equation}
such that we can write \eqref{eq:psd_cond_edg} fully per-edge in terms of $w_{ij}$:
\begin{equation}\label{eq:fin_wij}
\begin{aligned}
w_{ij}\leq \frac{1}{\alpha_{ij}}(\gamma_{ij} + \frac{1}{|\mathcal{N}_i^+| + |\mathcal{N}_i^-|} \beta_{i} + \frac{1}{|\mathcal{N}_j^+| + |\mathcal{N}_j^-|} \beta_{j})
\end{aligned}    
\end{equation}
At this point, by taking \eqref{eq:fin_wij} at equality and plugging it into \eqref{eq:adp_vdot}, the following expression for the Lyapunov time derivative is obtained:
\begin{equation}\label{eq:fin_adp_V}
\begin{aligned}
\dot{V}(\mathbf{s},\mathbf{k})= - \xi_{\mathbf{k}}^T \, \overline{\mathcal{L}}^{\symm} \, \xi_{\mathbf{k}} - \| \nabla_{\mathbf{k}}F(\mathbf{p},\mathbf{k})\|^2 \leq 0
\end{aligned}    
\end{equation}
where $\xi_{\mathbf{k}}$, unlike the original form in \cite{santilli2019distributed}, contains $\mathbf{k}$ the stacked vector of gains. 
From Theorem \ref{th:1} we know the first part of the above relation is negative semidefinite and from Theorem \ref{th:convx_F}, the function $F(\mathbf{p},\mathbf{k})$ is convex with respect to $\mathbf{k}$ for a fixed $\mathbf{p}$. Therefore, by invoking the LaSalle's invariance theorem the result follows.
\end{proof}

\begin{remark}
Note that a singularity situation with the control term $w_{ij}$ will never occur  for $\alpha_{ij}=0$ because $\dot{V}(\mathbf{s},\mathbf{k})$ is negative semi-definite and $V(\mathbf{s},\mathbf{k})$ is positive-definite as per Theorem \ref{th:gain_stab}. Next, note that with respect to the expression \eqref{eq:wij2} for $w_{ij}$, we assume that communication is available and bidirectional. These assumptions imply that each robot is aware of the cardinality of its incoming and outgoing neighbors which is realistic in the sense that each robot can identify its neighbor whom it is sensing (i.e. robots with unique AprilTags) and then communicate this information with its neighbor. Finally, note that the above analysis is trivially extendable to undirected graphs or symmetric interactions by repeating the analysis with bidirectional edges.
\end{remark}

In next section, we will modify the above adaptive gain tuning mechanism for MRSs with limited FOV to tackle external attacks or disturbances such that stable coordination of an MRS can be executed in the presence of sensor or actuator faults.

\section{Resilient Control Design}\label{sec:res_sec}
In this section, we derive an $H_{\infty}$ control protocol with the intention of making our proposed control framework (derived in the previous section) resilient to external attacks or disturbances which can render an MRS unstable. 
We know from the original form of the Gaussian potential function $\Psi_{ij}$ given in \cite{santilli2019distributed}, which represents the quality of interaction between robots, that the variances in measurements, $\sigma_x$ and $\sigma_y$,  are influenced by the sensor measurements in $\mathbf{p}=[\mathbf{x},\mathbf{y}]^T$. Consequently, the MRS control law is a direct function of $\nabla_{p_i}\Psi_{ij}(p_i,p_j)$ as also shown in equation \eqref{eq:u_i}. Therefore, it is reasonable to assume faults in the state variables $\mathbf{x}$ and $\mathbf{y}$ only. In this regard, we adopt the technique provided in \cite{chen2019resilient} to simultaneously tackle sensor and actuator faults and apply an $H_{\infty}$ control protocol, concluding with a novel proof of stability for our control framework.  We begin by defining the fault equations as
\begin{equation}\label{eq:sens_fault2}
\begin{aligned}
\hat{p}_i = p_i + \delta^p_i \quad \hat{u}_{p_i} = \Bar{u}_{p_i} + \delta^u_i
\end{aligned}    
\end{equation}
where $\hat{p}_i$ is the faulty sensor measurement, $p_i$ is the actual measurement which is unknown, $\delta^p_i$ is the unknown sensor fault, $\hat{u}_{p_i}$ is the faulty control input, $\Bar{u}_{p_i}$ is the desired control input such that the overall control action, as defined earlier, is $\Bar{u}_i=[\Bar{u}_{p_i}^T,\Bar{u}_{\theta_i} ]^T$ and $\delta^u_i$ is the unknown actuator fault induced by the attacker or external disturbances. 
Now, to proceed with designing the $H_{\infty}$ control protocol, we first define a \emph{linear} error equation as follows
\begin{equation}\label{eq:er}
\begin{aligned}
\Tilde{e}_i = \hat{p}_i - \Bar{p}_i - \hat{\delta}^p_i
\end{aligned}    
\end{equation}
where $\Bar{p}_i$ is an estimate of the non-corrupted measurement of the state $p_i$ and $\hat{\delta}^p_i$ is an estimate of the unknown sensor fault $\delta^p_i$. 
Then, the general $H_{\infty}$ control protocols are given by
\begin{equation}\label{eq:new_dyn}
\begin{aligned}
\Tilde{u}_{p_i}=\dot{\Bar{p}}_i = \Bar{u}_{p_i} + \mathcal{W}_i \quad \mathcal{W}_i =(F_1 + F_2)\Tilde{e}_i \quad \dot{\hat{\delta}}^p_i = -F_1 \Tilde{e}_i \\
\end{aligned}    
\end{equation}
where  $F_1$ and $F_2$ are the observer gains and the MRS  control law \eqref{eq:gen_con_law} is modified to $\Bar{u}_i=[\Tilde{u}_{p_i}^T,\Bar{u}_{\theta_i} ]^T$  . 
Before we provide the results from Theorem \ref{th:stab4} which will prove the boundedness of the error function in \eqref{eq:er} and consequently prove the stability of the complete MRS, we would like to provide the following Definition, Lemma, and assumptions for the general static output-feedback control design. 
\begin{definition}\label{def:hinf}
Define a linear time-invariant system $\dot{p}=\Bar{A}p+\Bar{B}u+\Bar{D}d$, $y=\Bar{C}p$, where $u$, $y$ and $d$ denote the system input, output, and disturbance, respectively. Define a performance output $o$ as $\|o\|^2=p^T \Bar{Q}p+ u^T \Bar{R}u$ for $\Bar{Q}\geq 0$ and $\Bar{R}>0$. The system $L_2$ gain is said to be bounded or attenuated by $\Gamma$ if the $L_2$ norms of $o$ and $d$ satisfy: $\frac{\int_{0}^{\infty} \|o\|^2 dt}{\int_{0}^{\infty} \|d\|^2 dt} = \frac{\int_{0}^{\infty} (p^T \Bar{Q}p + u^T \Bar{R}u) dt}{\int_{0}^{\infty} (d^Td) dt}\leq \Gamma^2$
\end{definition}

\begin{definition}\label{def:UUB}
The signal $z(t)$ is said to be uniformly ultimately bounded (UUB) with the ultimate bound $b$, if given positive constants $b$ and $c$, for every $a\in (0,c)$, there exists time $T(a,b)$. independent of time $t_0$, such that $\|z(t_0)\|\leq a \rightarrow \|z(t)\|\leq b, \forall t \geq t_0 + T$.
\end{definition}

\begin{lemma}(Theorem 1 in \cite{gadewadikar2006necessary})\label{lem:stab}
Assume that $(\Bar{A},\sqrt{\Bar{Q}})$ is detectable with $\Bar{Q}\geq 0$. Then, the system considered in Definition \ref{def:hinf} is output-feedback stabilizable with $L_2$ gain bounded by $\Gamma$, if and only if i) there exist matrices $\Bar{K}$, $M$, and $\Bar{P}$ such that 
\begin{equation}\label{eq:hinfcond}
\begin{aligned}
&\Bar{K}\Bar{C}=\Bar{R}^{-1}(\Bar{B}^T\Bar{P}+M)\\
&\Bar{P}\Bar{A}+\Bar{A}^T\Bar{P}+\Bar{Q}+\gamma^{-2}\Bar{P}\Bar{D}\Bar{D}^T\Bar{P}+ M^T \Bar{R}^{-1}M\\
&=\Bar{P}\Bar{B}\Bar{R}^{-1}\Bar{B}^T
\Bar{P}
\end{aligned}    
\end{equation}
and ii) $(\Bar{A},\Bar{B})$ is stabilizable and $(\Bar{A},\Bar{C})$ is detectable.
\end{lemma}
\begin{assumption}\label{asum:1}
The directed graph $\mathcal{G}$ contains a spanning tree with
the leader as its root.
\end{assumption}
\begin{assumption}\label{asum:2}
$\delta^p_i$ is unbounded but the derivative of $\dot{\delta}^p_i$ of the sensor fault in \eqref{eq:sens_fault2} and $\delta_i^u$ are bounded. 
\end{assumption}
Further, we define the estimation error for the sensor fault as $\Tilde{\delta}^p_i=\delta^p_i-\hat{\delta}^p_i$. Finally, the main result of this section is given by Theorem \ref{th:stab4}, where the error dynamics is shown to satisfy the \emph{static output-feedback control design} given by Definition \ref{def:hinf} and Lemma \ref{lem:stab} under Assumptions \ref{asum:1} and \ref{asum:2}, but the overall MRS stability analysis is conducted with respect to the original \emph{potential-based control framework}.

In Theorem \ref{th:stab4} we break the analysis into two parts. In the first part we show that even after the external sensor and actuator faults are induced, the error dynamics, $\dot{\Tilde{e}}_i$, is linear in nature. This allows us to apply the static output-feedback design technique, outlined in \cite{chen2019resilient}, that generally applies to time-invariant linear systems. The basic idea is that we will show first that the error $\Tilde{e}_i$ is bounded for all time, hence a time-invariant analysis is carried out. Once this is proven, we will move to the second part of the proof of Theorem \ref{th:stab4}, where we will analyze the Lyapunov stability in the context of UUB, as given by Definition \ref{def:UUB}, of the complete system using the global Lyapunov function given in \eqref{eq:adp_V} and the new, modified control laws in \eqref{eq:new_dyn}.

\begin{theorem}\label{th:stab4}
Suppose that the graph $\mathcal{G}$ satisfies Assumption \ref{asum:1}, the sensor and actuator faults satisfy Assumption \ref{asum:2}. Design $F_1$ and $F_2$ such that $\Bar{K}=[F_2^T,-F_1^T]$ satisfies \eqref{eq:hinfcond} in Lemma \ref{lem:stab}. Then the MRS topology control with limited FOV under the sensor and actuator faults \eqref{eq:sens_fault2} is solved by the control protocol \eqref{eq:new_dyn}. Moreover, the $L_2$ gains of the errors $\Tilde{e}_i$ and $\Tilde{\delta}^p_i$ are bounded in terms of the $L_2$ norms of disturbance $d_i$.
\begin{proof}
  Define $d=[d_i^T, d_2^T,...,d_n^T]^T$, $\Tilde{e}=[\Tilde{e}_1^T,\Tilde{e}_2^T,...,\Tilde{e}_n^T]^T$, and $\Tilde{\delta}^p=[\Tilde{\delta}^p_1,\Tilde{\delta}^p_2,...,\Tilde{\delta}^p_n]^T$. From $\dot{\mathbf{s}}(t) = \mathbf{u}(t)$ and \eqref{eq:new_dyn}, differentiating $\Tilde{e}_i$ in \eqref{eq:er}
with respect to time $t$ yields 
\begin{equation}\label{eq:err_dyn}
\begin{aligned}
&\dot{\Tilde{e}}_i= \dot{\hat{p}}_i - \dot{\Bar{p}}_i - \dot{\hat{\delta}}^p = - F_2 \Tilde{e}_i + \delta_i^u\\
\end{aligned}    
\end{equation}
Note that we consider $\dot{\hat{p_i}}=\hat{u}_{p_i}$ as it is reasonable to consider that the system is driven by the faulty control input.
Now we can write the above in a much more compact state-space form as
\begin{equation}\label{eq:hinfcond23}
\begin{aligned}
 \begin{bmatrix} \dot{\Tilde{e}}_i & \dot{\Tilde{\delta}}^p_i  \end{bmatrix}^T=A_{F_1}
 \begin{bmatrix} \Tilde{e}_i & \Tilde{\delta}^p_i  \end{bmatrix}^T + d_i
\end{aligned}    
\end{equation}
where $A_{F_1}=[a_{f_{ij}}^1]$ with $a_{f_{11}}^1= A-F_2$,$a_{f_{12}}^1= -A$,$a_{f_{21}}^1= F_1$,$a_{f_{22}}^1= 0$ and $d_i=[- \hat{p}_i + \Bar{p}_i +\dot{\delta}^p_i + \delta_i^u ,\quad \dot{\delta}^p_i]^T$. Note that the above analysis for a scenario excluding the actuator fault would look the same except the term $\delta_i^u$ would be eliminated in \eqref{eq:hinfcond23}. Also, at this point, the state-space form for the error dynamics in \eqref{eq:hinfcond23} is similar to the one in \cite{chen2019resilient}, except for the $d_i$ term.

In the following, we will show the stabilization of \eqref{eq:hinfcond23} can be achieved by appropriately designing $F_1$ and $F_2$. To do this, we transform \eqref{eq:hinfcond23} to the following full state output feedback control system 
\begin{equation}\label{eq:fdbck_sys}
\begin{aligned}
\dot{x}_T \triangleq \Bar{A}x_T + \Bar{B}u_T + \Bar{D}d_i, \quad y_T \triangleq \Bar{C}x_T
\end{aligned}    
\end{equation}
where $\Bar{A}=[A,-A;0,0]$, $\Bar{B}=\Bar{D}=[I_n,0;0,I_n]$ and $\Bar{C}=[I_n,0]$. Moreover, define the controller $u_T= -\Bar{K y_T}$, where $\Bar{K}=[K_1^T,K_2^T]^T$. Therefore, choosing $k_1=F_2$ and $K_2=-F_1$ makes \eqref{eq:fdbck_sys} equivalent to \eqref{eq:hinfcond23}. At this point after showing \eqref{eq:fdbck_sys} is equivalent to \eqref{eq:hinfcond23}, for the sake of brevity, the rest of the proof is not shown as it is exactly the same in Theorem 4 in \cite{chen2019resilient}. In brief, it is shown that by using  Definition \ref{def:hinf} and Lemma \ref{lem:stab} and as a result of the equivalence relation, $[\Tilde{e}^T, \tilde{\delta}^{p^T}]^T$ is \emph{bounded}. Next, the authors of \cite{chen2019resilient} use the static output-feedback technique to obtain a Lyapunov function for the error dynamics where they use the Lyapunov analysis to establish the UUB property. Using the UUB property, the authors claim that $[\Tilde{e}^T, \tilde{\delta}^{p^T}]^T$ is bounded for all time, as the static output-feedback technique is applied to linear time-invariant systems. We will now use the same UUB principle along with Lyapunov analysis and use the proved boundedness of error $\Tilde{e}$ to show boundedness of the system trajectories for the control protocols given in \eqref{eq:new_dyn}.
We will first begin with redoing the stability analysis, previously shown in Theorem \ref{th:gain_stab}, with the new MRS control protocols in \eqref{eq:new_dyn}. 
In \eqref{eq:adp_vdot}, we substitute the new control protocol $\Tilde{u}_{\Bar{p}}$ in the terms $\dot{s}$ and $\dot{\Bar{p}}_i$ where all $p_i$ and its associated terms in \eqref{eq:adp_vdot} is now replaced by $\Bar{p}_i$.
Therefore, restating the Lyapunov time derivative again as: 
\begin{equation}\label{eq:adp_vdot2}
\begin{aligned}
\dot{V}(\mathbf{s},\mathbf{k}) &= \nabla_{\mathbf{s}} \hat{V}(\mathbf{s},\mathbf{k})^T \dot{\mathbf{s}} + \nabla_{\mathbf{k}} \hat{V}(\mathbf{s},\mathbf{k})^T \dot{\mathbf{k}}\\
&+ \nabla_{\Bar{p}} F(\mathbf{\Bar{p}},\mathbf{k})^T \dot{\Bar{p}}+\nabla_{\mathbf{k}} F(\mathbf{\Bar{p}},\mathbf{k})^T \dot{\mathbf{k}}\\
\end{aligned}  
\end{equation}
From \eqref{eq:adp_vdot2}, the expansion of the  term $\nabla_{\Bar{p}} F(\mathbf{\Bar{p}},\mathbf{k})^T \dot{\Bar{p}}$ is as shown below
\begin{equation}\label{eq:adp_hatp}
\begin{aligned}
 &\nabla_{\Bar{p}} F(\mathbf{\Bar{p}},\mathbf{k})^T \dot{\Bar{p}} =\nabla_{\Bar{p}} F(\mathbf{\Bar{p}},\mathbf{k})^T (\Bar{u}_{\mathbf{\Bar{p}}} + (I_{2n} \kron F_1 + F_2)\Tilde{e})= \\
 &\nabla_{\Bar{p}} F(\mathbf{\Bar{p}},\mathbf{k})^T \nabla_{\Bar{p}+}\hat{V}(\Bar{p},\mathbf{k})  + \nabla_{\Bar{p}} F(\mathbf{\Bar{p}},\mathbf{k})^T(I_{2n} \kron (F_1 + F_2))\Tilde{e}
\end{aligned}    
\end{equation}
Similar to \eqref{eq:psd_cond},  the term $\nabla_{\Bar{p}} F(\mathbf{\Bar{p}},\mathbf{k})^T \nabla_{\Bar{p}+}\hat{V}(\Bar{p},\mathbf{k})$ is contained in $\mathbf{w}$ and the only new term we need to include in the final equation of the Lyapunov time derivative as in \eqref{eq:fin_adp_V} is $\nabla_{\Bar{p}} F(\mathbf{\Bar{p}},\mathbf{k})^T(I_{2n} \kron (F_1 + F_2))\Tilde{e}$.
Similarly, for the term $\nabla_{\mathbf{s}} \hat{V}(\mathbf{s},\mathbf{k})^T \dot{\mathbf{s}}$ in the Lyapunov time derivative \eqref{eq:adp_vdot2}, we will end up with the term $\nabla_{\mathbf{\Bar{p}}} \hat{V}(\mathbf{\Bar{p}},\mathbf{k})^T(I_{2n} \kron (F_1 + F_2))\Tilde{e} $, where $I_{2n} \in \mathbb{R}^{2n \times 2n}$ is an identity matrix. Therefore, the final form of the Lyapunov time derivative in \eqref{eq:adp_vdot2},  in comparison to the old Lyapunov time derivative given in \eqref{eq:fin_adp_V}, looks like 
\begin{equation}\label{eq:fin_adp_Vnew}
\begin{aligned}
&\dot{V}(\mathbf{s},\mathbf{k})= - \xi_{\mathbf{k}}^T \, \overline{\mathcal{L}}^{\symm} \, \xi_{\mathbf{k}} - \| \nabla_{\mathbf{k}}F(\mathbf{\Bar{p}},\mathbf{k})\|^2 +\ldots\\
&(\nabla_{\mathbf{\Bar{p}}} \hat{V}(\mathbf{\Bar{p}},\mathbf{k})^T+\nabla_{\Bar{p}} F(\mathbf{\Bar{p}},\mathbf{k})^T) (I_{2n} \kron (F_1 + F_2))\Tilde{e}
\end{aligned}    
\end{equation}
Next, we will proceed to establish the UUB property for the above Lyapunov time derivative  to show that the system is stable.
The Lyapunov time derivative in \eqref{eq:fin_adp_Vnew} can also be written as
\begin{equation}\label{eq:fin_adp_Vnew2}
\begin{aligned}
&\dot{V}(\mathbf{s},\mathbf{k})\leq  - \delta_{min}(\overline{\mathcal{L}}^{\symm}) \|\xi_{\mathbf{k}} \|^2- \| \nabla_{\mathbf{k}}F(\mathbf{\Bar{p}},\mathbf{k})\|^2 + \ldots\\ 
&\|\nabla_{\mathbf{\Bar{p}}} \hat{V}(\mathbf{\Bar{p}},\mathbf{k})^T+\nabla_{\Bar{p}} F(\mathbf{\Bar{p}},\mathbf{k})^T\| \|(I_{2n} \kron (F_1 + F_2))\| \|\Tilde{e}\|
\end{aligned}    
\end{equation}
where $\delta_{min}(\overline{\mathcal{L}}^{\symm})=0$ is the smallest singular value of the positive semi-definite matrix $\overline{\mathcal{L}}^{\symm}$. By taking  $\tilde{e}_o$ as the upper bound of $\|\Tilde{e}\|$, we can simplify \eqref{eq:fin_adp_Vnew2} as
\begin{equation}\label{eq:fin_adp_Vnew3}
\begin{aligned}
&\dot{V}(\mathbf{s},\mathbf{k})\leq- \| \nabla_{\mathbf{k}}F(\mathbf{\Bar{p}},\mathbf{k})\|^2 + \ldots\\ 
&\|\nabla_{\mathbf{\Bar{p}}} \hat{V}(\mathbf{\Bar{p}},\mathbf{k})^T+\nabla_{\Bar{p}} F(\mathbf{\Bar{p}},\mathbf{k})^T\| \|(I_{2n} \kron (F_1 + F_2))\| \Tilde{e}_o
\end{aligned}    
\end{equation}
It then immediately follows that 
\begin{equation}\label{eq:fin_adp_Vnew4}
\begin{aligned}
&\dot{V}(\mathbf{s},\mathbf{k})< 0,\ldots \\ 
&\forall \underbrace{\frac{ \| \nabla_{\mathbf{k}}F(\mathbf{\Bar{p}},\mathbf{k})\|^2 }{\|\nabla_{\mathbf{\Bar{p}}} \hat{V}(\mathbf{\Bar{p}},\mathbf{k})^T+\nabla_{\Bar{p}} F(\mathbf{\Bar{p}},\mathbf{k})^T\| \|(I_{2n} \kron (F_1 + F_2))\|}}_{\Omega}>\Tilde{e}_o
\end{aligned}    
\end{equation}
Therefore, using the idea of UUB, the time derivative of $V$ is negative outside a compact set defined as $\Lambda_{\Tilde{e}_o}=\{\Omega \leq \Tilde{e}_o\}$, or equivalently, all the solutions that start outside $\Lambda_{\Tilde{e}_o}$ will enter this set within finite time, and will remain inside the set for all future times. This can be further explained as follows. Choose $\chi>\frac{\Tilde{e}_o^2}{2}$. Then all solutions that start in the set 
\begin{equation}\label{eq:fin_adp_Vnew5}
\begin{aligned}
\Lambda_{\chi}=\{V(\mathbf{s},\mathbf{k}) \leq \chi\} \supset \Lambda_{\Tilde{e}_o}
\end{aligned}    
\end{equation}
will remain inside the set, for all future times. This occurs because $\dot{V}$ is negative on the boundary. Hence, the solutions are \emph{uniformly bounded}. Now, the ultimate bounds of these solutions can also be found by choosing $\partial$ such that
\begin{equation}\label{eq:fin_adp_Vnew6}
\begin{aligned}
\frac{\Tilde{e}_o^2}{2}<\partial<\chi
\end{aligned}    
\end{equation}
Then $\dot{V}$ is negative in the annulus set $\{\partial \leq V(\mathbf{s},\mathbf{k})\leq\chi \}$, which implies that in this set $V(\mathbf{s},\mathbf{k})$ will decrease monotonically in time, until the solution enters the set $\{ V(\mathbf{s},\mathbf{k})\leq\partial \}$. Therefore, from this time on the trajectory of $V(\mathbf{s},\mathbf{k})$ will remain in the set, because $\dot{V}$ is negative on the boundary $V(\mathbf{s},\mathbf{k}) = \partial$. Hence, this proves the UUB property. 
\end{proof}

\begin{remark}
The above analysis was broken into two parts: i) showing the boundedness of error using the static output-feedback technique, ii) using (i) with Lyapunov analysis to establish UUB property for the new control protocols in \eqref{eq:new_dyn}. Interestingly, note that we apply the static output-feedback technique to a potential-based limited FOV control framework for MRS where the control law is non-linear in nature, but the error dynamics turn out to be linear as the non-linearities cancel out in the analysis.
\end{remark}

\end{theorem}

\section{Simulation Results}\label{sec:simulations}
In this section, we discuss results from a \emph{realistic} ROS Gazebo simulation. The simulation is executed over $250$ seconds with six robots flying at the same altitude. Robot $4$ is the leader with exogenous input, for a leader-follower scenario as also depicted in Figure~\ref{fig:gaz_pose} and the video\footnote{Refer to the submitted video for  MATLAB simulation results as well as Gazebo simulation and comparison of two scenarios: i)MRS with adaptive gains ii)MRS without adaptive gains.Code: \href{https://github.com/mukhe027/catkin\_ws.git}{https://github.com/mukhe027/catkin\_ws.git}} submitted with this paper. 
 \begin{figure}[t!]
\centering
    \includegraphics[width=0.4\textwidth]{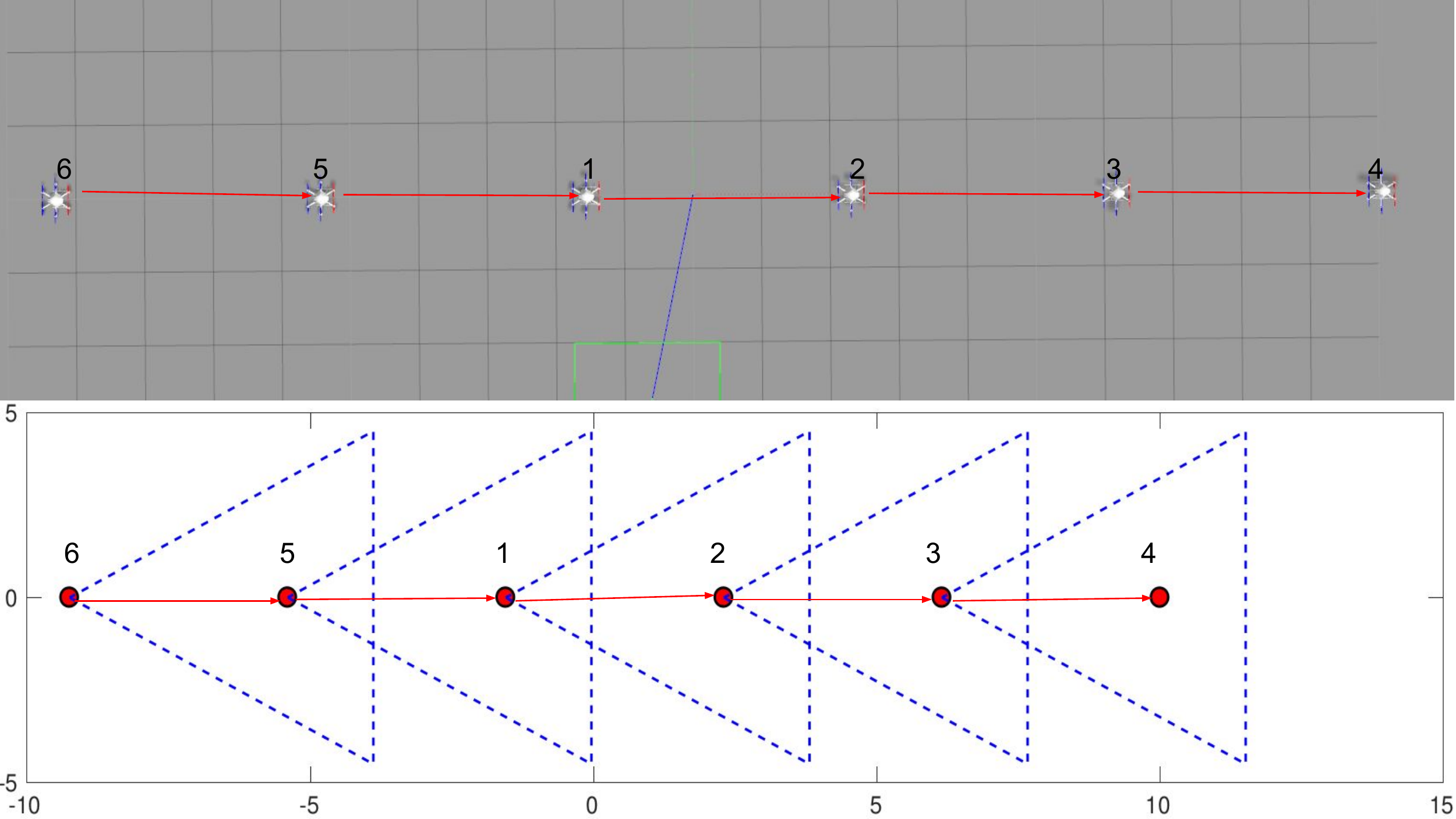}
    \caption{Robots in a spanning tree orientation (leader robot $4$) in ROS C++ Gazebo environment with FOV of each robot represented in corresponding MATLAB plot. }
    \label{fig:gaz_pose}
\end{figure}
\def\figsize{0.89}
\def\figsizee{0.95}
\def\figsizeee{0.92}
\def\subfigsize{0.5}
\def\spacereduction{-0.15cm}

\begin{figure}[t!]
    \centering
    \includegraphics[width=0.45\textwidth]{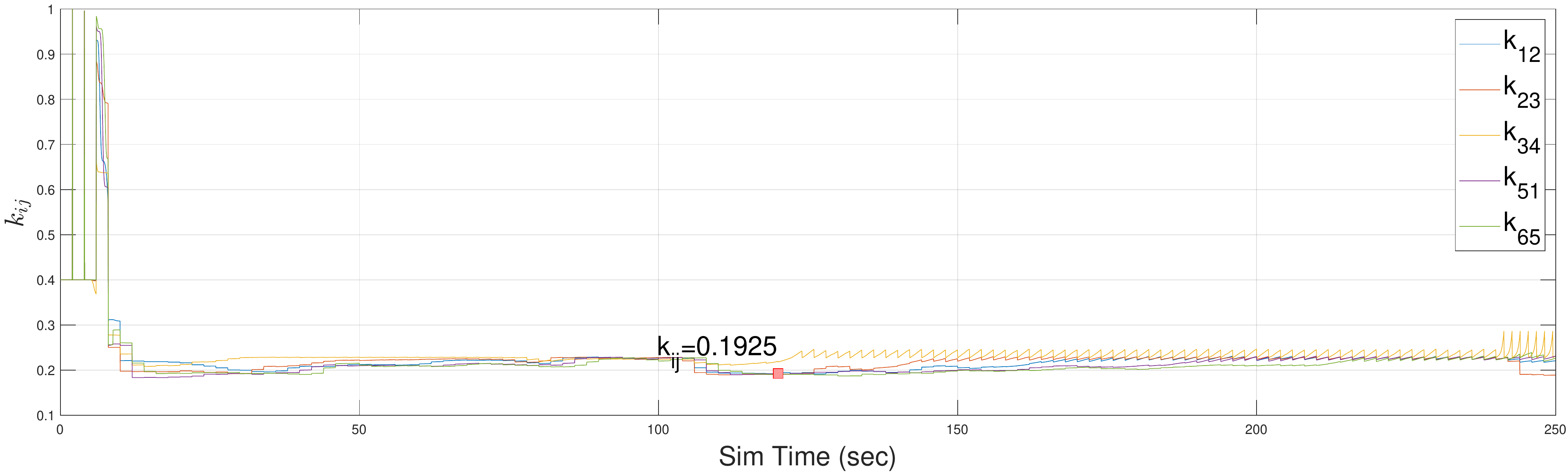}
    \caption{$k_{ij}$ computed using adaptive law in equation \eqref{eq:adp_law}. }\label{fig:kij}
\end{figure}

\begin{figure}[t!]
    \centering
    \includegraphics[width=0.45\textwidth]{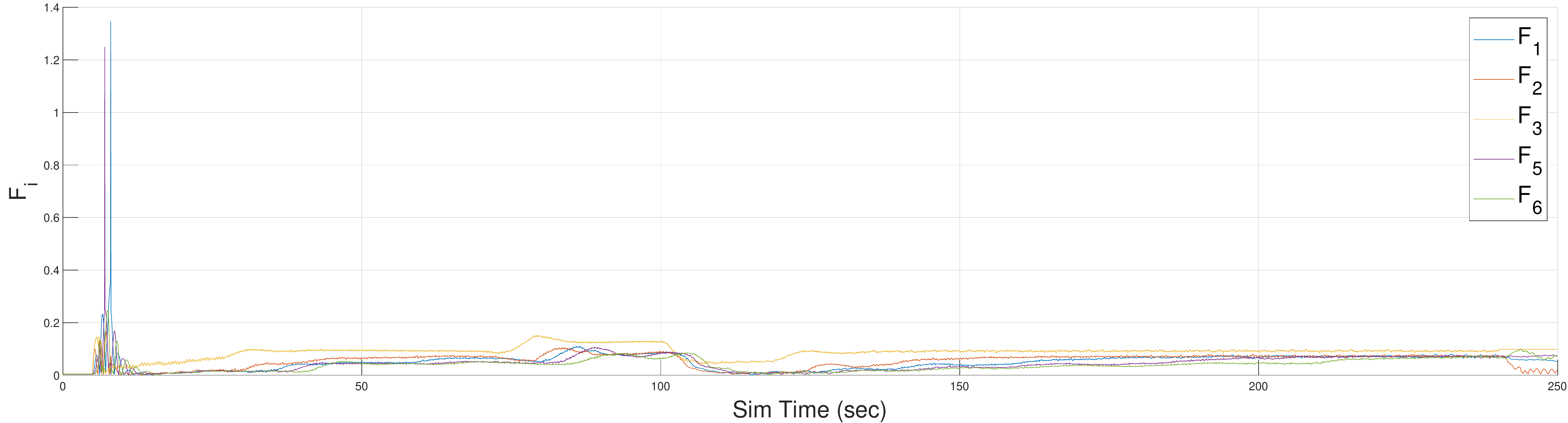}
    \caption{Plot of $F_i(p,k)$, the pairwise cost function encoding the deviation of pairwise interactions from its nominal model. }\label{fig:F}
\end{figure}

Figure~\ref{fig:kij} shows the adaptive gain $k_{ij}$ solved using adaptive laws \eqref{eq:adp_law}. Corresponding to the the computed $k_{ij}$, Figure~\ref{fig:F} shows the plot of the per robot cost $F_i$. Clearly, the $F_i$ values converge to values close to zero, the optimal solution representing no deviation from the nominal model. However, any change in the actuation (i.e., change in speed or heading while in motion) of the robots leads to the $F_i$ values to deviate slightly away from zero but still remain sufficiently close to the optimal value for a stable leader-follower scenario. 

\def\figsize{0.89}
\def\figsizee{0.95}
\def\figsizeee{0.92}
\def\subfigsize{0.5}
\def\spacereduction{-0.15cm}

\begin{figure}
\centering
\begin{subfigure}[b]{0.43\textwidth}
   \includegraphics[width=1\linewidth]{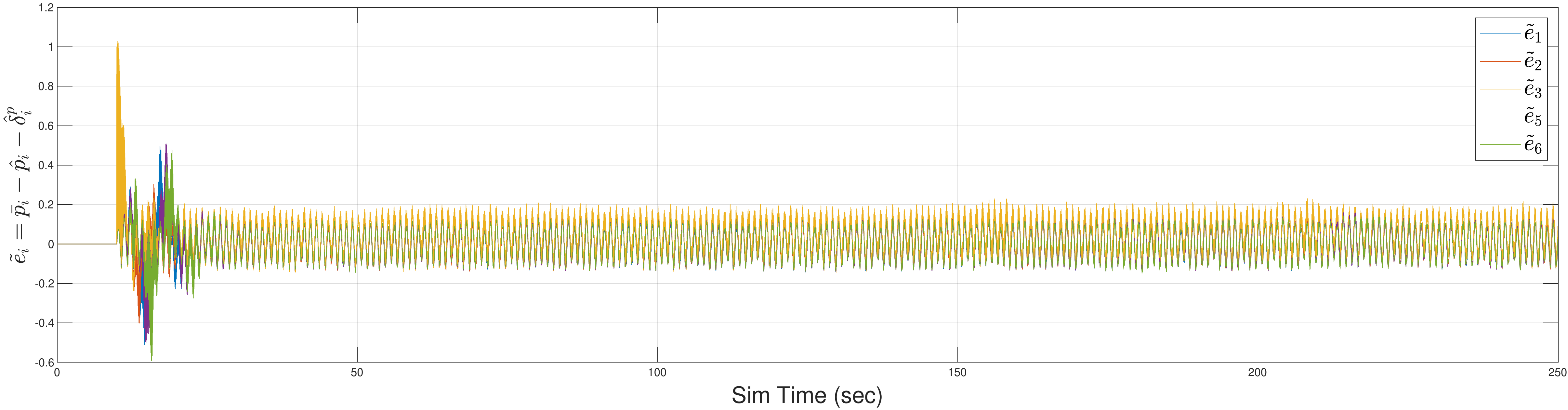}
   \caption{}
   \label{fig:er} 
\end{subfigure}

\begin{subfigure}[b]{0.43\textwidth}
   \includegraphics[width=1\linewidth]{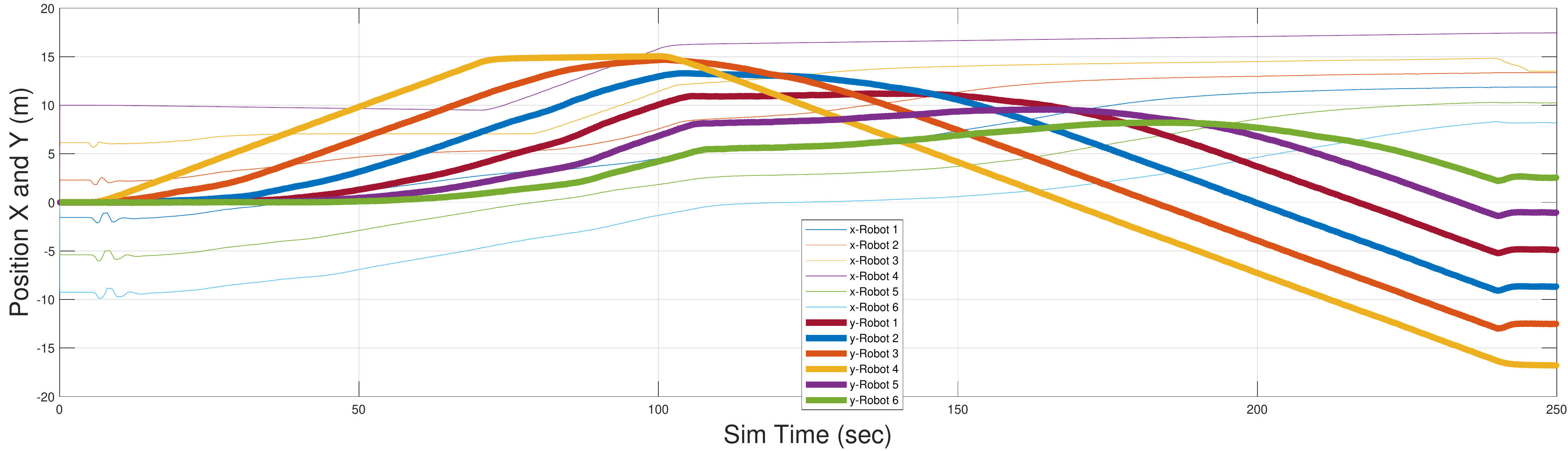}
   \caption{}
   \label{fig:robot_dist}
\end{subfigure}

\caption[Two numerical solutions]{(a) Plot of error dynamics $\Tilde{e}_i$ for each robot as given by \eqref{eq:er}. (b) Plot of measured positions $(x,y)$ of in MRS for complete simulation time of $250s$.}
\end{figure}
In our simulations, robot $3$ is induced with an unbounded sensor fault $0.2t$ and all robots (except leader) are induced with bounded actuation fault $1.5sin(2\pi t)$ where both comply with Assumption \ref{asum:2}. The yellow plot in Figure~\ref{fig:er} shows the error dynamics of robot $3$, where clearly the error induced due to sensor fault in Robot $3$ is attenuated. Similarly the oscillatory, bounded actuator fault, which too is attenuated, is visible in the oscillation of the error dynamics of all the robots. Now, Figure~\ref{fig:robot_dist} shows the measured positions, ($x$,$y$),  of the robots with no indication of the faults having any impact on the trajectories of the robots implying that the \emph{global} topology control objective of leader-following was attained despite of the faults, while also adaptive tuning interactions to match desired models of robot interaction. 

\section{Conclusion}\label{sec:conclusion}
We derived an adaptive control law capable of maintaining the interactions between robots in MRS such that pairwise ``prescribed'' performance is achieved which is then coupled with $H_{\infty}$ control protocols for resiliency towards additive sensor and actuator faults. 
Realistic ROS Gazebo simulations are provided to support the theory developed. In this regard, we have shown that that our methods can reduce the gap between
easy to analyze multi-robot control methods using the \emph{potential-based control framework} and real-world
deployments.

\bibliographystyle{IEEEtran}
\bibliography{arxiv}

\end{document}